\documentclass[11pt,a4paper]{article}

\usepackage{iclr2027_conference,times}
\iclrfinalcopy  

\usepackage{amsmath,amsfonts,bm}

\def\eqref#1{equation~\ref{#1}}

\def\1{\bm{1}}

\DeclareMathAlphabet{\mathsfit}{\encodingdefault}{\sfdefault}{m}{sl}
\SetMathAlphabet{\mathsfit}{bold}{\encodingdefault}{\sfdefault}{bx}{n}

\usepackage{hyperref}
\usepackage{url}

\usepackage[utf8]{inputenc} 
\usepackage[T1]{fontenc}    
\usepackage{hyperref}       
\usepackage{url}            
\usepackage{booktabs}       
\usepackage{amsfonts}       
\usepackage{nicefrac}       
\usepackage{microtype}      
\usepackage{xcolor}         

\usepackage{amsmath}
\usepackage{amssymb}
\usepackage{mathtools}
\usepackage{amsthm}
\usepackage{comment}
\usepackage{graphicx}
\usepackage{subcaption}

\theoremstyle{plain}
\newtheorem{theorem}{Theorem}[section]
\newtheorem{proposition}[theorem]{Proposition}
\newtheorem{lemma}[theorem]{Lemma}

\theoremstyle{definition}

\newtheorem{assumption}[theorem]{Assumption}
\theoremstyle{remark}

\usepackage{amsthm}

\usepackage{multirow}

\usepackage{titletoc}
\usepackage{xcolor}         

\usepackage{tabularx}
\usepackage{graphicx}
\usepackage{colortbl}

\definecolor{LightMint}{RGB}{170,235,205}

\definecolor{LightRed}{RGB}{255,182,193} 
\definecolor{LightBlue}{RGB}{173, 216, 230}
\definecolor{LightGreen}{RGB}{144, 238, 144}

\definecolor{LightMint}{RGB}{213,245,230}
\definecolor{LightRed}{RGB}{255,219,224}
\definecolor{LightBlue}{RGB}{214,236,243}
\definecolor{LightGreen}{RGB}{200,247,200}

\colorlet{shadecolor}{LightGreen}
\colorlet{shadecolor}{LightBlue}
\colorlet{shadecolor}{LightRed}

\usepackage{enumitem}
\setlist[enumerate]{leftmargin=3em}

\usepackage{arydshln} 
\usepackage{algorithm}
\usepackage{algpseudocode}

\usepackage{graphicx}
\usepackage{wrapfig}

\title{CTP-FL: Common-Trajectory Gradient Prediction for Federated Learning}

\author{ Junkang Liu\textsuperscript{1}\\
Tianjin University\\
	\texttt{\{junkangliukk\}@gmail.com}
}

\begin{document}

	\maketitle

\begin{abstract}
	Communication-efficient federated optimization commonly spends several
	gradient evaluations between server updates. Existing local-update
	methods use this computation to advance an independent model on each
	client. Under heterogeneous data, however, these models evaluate
	gradients at different locations, making the aggregated update
	difficult to interpret as a gradient of the global objective.
	We study an alternative use of the same computation budget:
	\emph{evaluate the global objective along a shared, predicted path}.
	We propose Common-Trajectory Predictive Federated Learning
	(\texttt{CTP-FL}). At each round, all clients construct the same
	sequence of query points from the current global model and the
	previous aggregated direction, evaluate $K$ stochastic gradients
	along this sequence, and upload their average. The server then
	performs a single global update. Thus, \texttt{CTP-FL} uses $K$
	mini-batch gradients per client and one model-sized vector in each
	communication direction, matching the per-round computation and
	communication of full-participation FedAvg-M.
	Shared query points make the aggregated direction an unbiased
	estimator of the average \emph{global} gradient along the predicted
	path. The remaining discrepancy from the gradient at the current
	model is controlled by the path length, without assuming bounded
	client-gradient dissimilarity or bounded gradients. For smooth
	non-convex objectives, we establish an
	$\mathcal{O}\!\left(
	\sqrt{L\Delta\sigma^2/(NKR)}+L\Delta/R
	\right)$
	average-stationarity bound under full participation. The analysis
	isolates a testable trade-off: extending the prediction path provides
	more forward-looking gradient information but increases its
	displacement bias.
\end{abstract}

\section{Introduction}
\label{sec:introduction}

Federated learning trains a shared model from decentralized data
while limiting communication between clients and a server.
Federated Averaging (FedAvg) addresses communication cost by
performing multiple client updates before each aggregation
\citep{mcmahan2017communication}. This structure has shaped a
large body of work on local optimization
\citep{stich2019local}. Yet the role of the computation performed
between communications deserves closer examination: $K$ local
gradient evaluations need not be used to produce $K$ independent
client trajectories.

\paragraph{Where does the difficulty arise?}
Let $F(\boldsymbol{x})=N^{-1}\sum_{i=1}^{N}F_i(\boldsymbol{x})$.
In a local-update method, every client begins a round at the same
$\boldsymbol{x}^r$, but its later gradients are evaluated at
client-specific points $\boldsymbol{x}_i^{r,k}$. Even when all
clients participate, the average gradient at these different
points is generally not a gradient of $F$ at any single common
point. This distinction is particularly relevant under
heterogeneous client objectives: differences in local gradients
alter the trajectories, which in turn alter where subsequent
gradients are measured. SCAFFOLD addresses this client drift
with control variates \citep{karimireddy2020scaffold}.
FedAvg-M instead introduces a global momentum direction into
local updates and proves full-participation convergence without
a bounded data-heterogeneity assumption
\citep{cheng2024momentum}. These results establish that
heterogeneity can be handled without assuming similar client
gradients. They also motivate a different question: can the
gradient locations themselves be coordinated without increasing
per-round communication?

\paragraph{A shared path in place of independent trajectories.}
We propose Common-Trajectory Predictive Federated Learning
(\texttt{CTP-FL}). At round $r$, the server model
$\boldsymbol{x}^r$ and previous aggregate direction
$\boldsymbol{v}^r$ define $K$ shared query points
\begin{equation}
	\label{eq:intro_ctp_path}
	\boldsymbol{z}_k^r
	=
	\boldsymbol{x}^r-t_k\boldsymbol{v}^r,
	\qquad
	t_k\in[0,h],
	\qquad k=0,\ldots,K-1.
\end{equation}
Every client evaluates one mini-batch gradient at each
$\boldsymbol{z}_k^r$ and uploads only their average. The server
averages these client vectors to obtain
$\boldsymbol{v}^{r+1}$ and updates
$\boldsymbol{x}^{r+1}
=\boldsymbol{x}^r-\gamma\boldsymbol{v}^{r+1}$.
The previous direction can be recovered from consecutive global
models, so it need not be transmitted as an additional vector.

The distinction from local SGD and FedAvg-M is structural.
\texttt{CTP-FL} does not execute $K$ local model updates:
the $K$ computations probe a common predicted path before
one server update. It therefore gives up the adaptive local
trajectories that can make local training effective, while
ensuring that clients evaluate the same sequence of model
parameters. The prediction path is also different from
Lookahead's interpolation between fast and slow weights
\citep{zhang2019lookahead}: here it specifies where distributed
gradients are measured, rather than combining locally updated
weights.

\paragraph{What does the shared path buy?}
Conditioned on the history of round $r$, independent unbiased
mini-batch gradients give
\begin{equation}
	\label{eq:intro_ctp_direction}
	\mathbb{E}[\boldsymbol{v}^{r+1}\mid\mathcal F_r]
	=
	\frac{1}{K}\sum_{k=0}^{K-1}
	\nabla F(\boldsymbol{z}_k^r).
\end{equation}
This identity holds for arbitrarily different client
objectives under full participation: client-gradient
differences cancel at every common query point. The direction
is not, in general, an unbiased estimate of
$\nabla F(\boldsymbol{x}^r)$. If each $F_i$ is $L$-smooth,
its displacement bias satisfies
\begin{equation}
	\label{eq:intro_ctp_bias}
	\left\|
	\frac{1}{K}\sum_{k=0}^{K-1}
	\nabla F(\boldsymbol{z}_k^r)
	-\nabla F(\boldsymbol{x}^r)
	\right\|
	\le
	\frac{L}{K}\sum_{k=0}^{K-1}
	t_k\|\boldsymbol{v}^r\|.
\end{equation}
For the uniform path from $0$ to $h$, the right-hand side
equals $Lh\|\boldsymbol{v}^r\|/2$ as an upper bound.
Equations~\eqref{eq:intro_ctp_direction} and
\eqref{eq:intro_ctp_bias} expose the method's central
trade-off: $h$ determines how far ahead the clients measure
the global gradient field, while also controlling the bias
relative to the current iterate.

\paragraph{Guarantee and scope.}
Under $L$-smoothness, unbiased stochastic gradients with
variance at most $\sigma^2$, and full client participation,
we prove that the choice $h=\gamma$ with an appropriate
server stepsize satisfies
\begin{equation}
	\label{eq:intro_ctp_rate}
	\frac{1}{R}\sum_{r=0}^{R-1}
	\mathbb{E}\|\nabla F(\boldsymbol{x}^r)\|^2
	=
	\mathcal{O}\!\left(
	\sqrt{\frac{L\Delta\sigma^2}{NKR}}
	+
	\frac{L\Delta}{R}
	\right),
\end{equation}
where $\Delta=F(\boldsymbol{x}^0)-\inf_{\boldsymbol{x}}
F(\boldsymbol{x})$. The proof requires neither bounded
gradients nor bounded client-gradient dissimilarity.
This rate is of the same order as the full-participation
FedAvg-M guarantee; it does not establish a theoretical
acceleration over FedAvg-M. Importantly, the cancellation
in~\eqref{eq:intro_ctp_direction} relies on full
participation. With client sampling, an additional
sampling-variance term appears, and the present algorithm
does not provide a heterogeneity-independent
partial-participation guarantee.

Our contributions are:
\begin{itemize}
	\item We introduce a communication-neutral allocation
	of $K$ client gradient evaluations to a shared predicted
	trajectory, with one model-sized uplink and downlink
	vector per client per round.
	\item We characterize the resulting direction as the
	average global gradient along the shared trajectory and
	identify its displacement bias separately from
	stochastic-gradient noise.
	\item We prove the non-convex
	full-participation guarantee in
	\eqref{eq:intro_ctp_rate} without bounded gradients
	or a bounded data-heterogeneity assumption.
\end{itemize}

\section{Related Work}
\label{sec:related_work}

\paragraph{Local updates and communication efficiency.}
FedAvg reduces communication by allowing clients to take multiple
optimization steps before model averaging
\citep{mcmahan2017communication}. Local SGD provides a broader
distributed-optimization perspective on this strategy
\citep{stich2019local}. Its benefit, however, depends on how the
saved communication interacts with the errors introduced by
independent local trajectories. Woodworth et al.\
\citep{woodworth2020local} show that local SGD and mini-batch
SGD do not uniformly dominate one another across objective
classes. This comparison is especially relevant to our setting:
\texttt{CTP-FL} allocates $K$ gradient evaluations per client to
$K$ \emph{shared query points}, rather than to $K$ successive
local model updates. When its prediction span is zero,
\texttt{CTP-FL} reduces to a synchronized mini-batch-gradient
baseline. Thus, any benefit of prediction must be measured
against both this baseline and strong local-update methods at
matched computation and communication budgets.

\paragraph{Data heterogeneity and client drift.}
Under heterogeneous client objectives, local models can follow
different trajectories even when they start from the same
global model. FedProx regularizes local optimization around
the current server model \citep{li2020fedprox}, while
SCAFFOLD uses client and server control variates to correct
client drift \citep{karimireddy2020scaffold}. Mime combines
control variates with server-side optimizer statistics to
approximate the behavior of centralized optimizers during
local training \citep{karimireddy2021mime}. FedAvg-M injects
a global momentum direction into local updates and establishes
full-participation convergence without assuming bounded
client-gradient dissimilarity; it also develops momentum
variants of SCAFFOLD for partial participation
\citep{cheng2024momentum}. Hence, the absence of a bounded
heterogeneity assumption is not, by itself, a distinguishing
claim for \texttt{CTP-FL}. Our distinction is where gradients
are evaluated: at local, client-dependent iterates in these
methods, and at the same predicted positions on every client
in \texttt{CTP-FL}. Under full participation, averaging the
latter gradients gives an estimator of the global gradient
averaged along the shared path. The remaining bias is governed
by the path's displacement from the current model.

\paragraph{Prediction and shared gradient evaluation.}
Prediction has also been used to organize optimization across
multiple steps. Lookahead maintains fast and slow weights and
interpolates between them after several optimizer updates
\citep{zhang2019lookahead}. \texttt{CTP-FL} uses the previous
aggregated direction for a different purpose: it specifies
the locations of the current round's gradient evaluations,
without advancing separate client models along that path.
The resulting direction is generally biased relative to
$\nabla F(\boldsymbol{x}^r)$, even though it is an unbiased
estimate of the average global gradient at the queried
positions. Our analysis makes this displacement bias explicit
and bounds it using smoothness and the prediction span.
The present guarantee applies to full participation;
subsampling clients introduces additional sampling variance
that the basic \texttt{CTP-FL} update does not remove.

\section{The Proposed  Algorithm}
\vspace{-2mm}
\label{sec:method}
\subsection{Problem Setup}
\vspace{-2mm}
FL seeks to learn a global model collaboratively over clients by minimizing the population risk:
\vspace{-2mm}
\begin{equation}
	f_i(\boldsymbol{x}) 
	:= \mathbb{E}_{\xi_i \sim \mathcal{D}_i}
	\bigl[F_i(\boldsymbol{x}; \xi_i)\bigr], \qquad f(\boldsymbol{x}) := \frac{1}{N} \sum_{i=1}^N f_i(\boldsymbol{x}).
	\label{eq:objective}
\end{equation}
The function $F_i$ is the loss function on client $i$. $\mathbb{E}_{\xi_i \sim \mathcal{D}_i}[\cdot]$ denotes conditional expectation with respect to  the sample $\xi_i$. $N$ is the number of clients, and 
$\boldsymbol{x}$ is global model.

\subsection{From Local Trajectories to a Shared Gradient Path}
\label{subsec:ctp_motivation}

Consider the federated objective
\begin{equation}
	\label{eq:ctp_problem}
	\min_{\boldsymbol{x}\in\mathbb{R}^d}
	F(\boldsymbol{x})
	\coloneqq
	\frac{1}{N}\sum_{i=1}^{N}F_i(\boldsymbol{x}),
\end{equation}
where $F_i$ is the objective of client $i$.
A communication round typically allows each client to compute
$K$ stochastic gradients before uploading one model-sized vector.
In local-update methods, these gradients are evaluated along
client-specific iterates $\boldsymbol{x}_i^{r,k}$. Consequently,
even under full participation, their aggregate takes the form
\begin{equation}
	\label{eq:ctp_local_mismatch}
	\frac{1}{N}\sum_{i=1}^{N}
	\nabla F_i(\boldsymbol{x}_i^{r,k}),
\end{equation}
which generally differs from
$\nabla F(\boldsymbol{x})$ at any common model
$\boldsymbol{x}$. The mismatch arises from the
\emph{locations of gradient evaluation}: heterogeneity first
separates the client trajectories, after which subsequent
gradients are measured at different locations.

\texttt{CTP-FL} uses the same $K$ gradient evaluations to query
a \emph{shared predictive path}. Instead of constructing
$N$ local trajectories, all clients evaluate their objectives
at the same sequence of points. This makes the client average
at each point an estimate of the corresponding global gradient.
The remaining approximation is explicit: the queried global
gradients are displaced from the current server model.

\subsection{Common-Trajectory Predictive Federated Learning}
\label{subsec:ctp_update}

At the start of communication round $r$, let
$\boldsymbol{x}^r$ be the global model and
$\boldsymbol{v}^r$ the direction aggregated in the previous
round. Set $\boldsymbol{v}^0=\boldsymbol{0}$.
Given a prediction span $h\ge0$, define
\begin{equation}
	\label{eq:ctp_query_points}
	t_k=
	\begin{cases}
		\dfrac{kh}{K-1},&K>1,\\[4pt]
		0,&K=1,
	\end{cases}
	\qquad
	\boldsymbol{z}_k^r
	=
	\boldsymbol{x}^r-t_k\boldsymbol{v}^r,
	\quad k=0,\ldots,K-1.
\end{equation}
The points $\boldsymbol{z}_k^r$ are identical across clients
and are fixed before the current round's mini-batches are
sampled. Client $i$ computes
\begin{equation}
	\label{eq:ctp_client_update}
	\boldsymbol{u}_i^r
	=
	\frac{1}{K}\sum_{k=0}^{K-1}
	g_i(\boldsymbol{z}_k^r;\xi_{i,k}^r),
\end{equation}
where $g_i(\boldsymbol{z}_k^r;\xi_{i,k}^r)$ is a stochastic
gradient of $F_i$ at the shared point
$\boldsymbol{z}_k^r$. The server aggregates
\begin{equation}
	\label{eq:ctp_server_update}
	\boldsymbol{v}^{r+1}
	=
	\frac{1}{N}\sum_{i=1}^{N}\boldsymbol{u}_i^r,
	\qquad
	\boldsymbol{x}^{r+1}
	=
	\boldsymbol{x}^r-\gamma\boldsymbol{v}^{r+1}.
\end{equation}
Algorithm~\ref{algorithm_CTP_FL} summarizes the procedure.
The $K$ query points are gradient-evaluation locations;
they are \emph{not} successive local model updates.

The previous direction requires no separate downlink vector
under full participation. A client retaining
$\boldsymbol{x}^{r-1}$ recovers it from
\begin{equation}
	\label{eq:ctp_direction_recovery}
	\boldsymbol{v}^{r}
	=
	\frac{\boldsymbol{x}^{r-1}-\boldsymbol{x}^{r}}{\gamma},
	\qquad r\ge1.
\end{equation}
Thus, each round uses $K$ mini-batch gradient evaluations
per client, one $d$-dimensional uplink vector
$\boldsymbol{u}_i^r$, and one $d$-dimensional downlink
model $\boldsymbol{x}^r$. The protocol has the same
per-round gradient count and vector communication as
full-participation FedAvg-M, although its computation
does not perform local SGD steps.

\subsection{Why Shared Queries Change the Error Structure}
\label{subsec:ctp_insight}

Let $\mathcal F_r$ contain the history before the current
mini-batches are drawn. Under unbiased stochastic gradients,
\begin{equation}
	\label{eq:ctp_path_identity}
	\mathbb E[
	\boldsymbol{v}^{r+1}\mid\mathcal F_r]
	=
	\underbrace{
		\frac{1}{K}\sum_{k=0}^{K-1}
		\nabla F(\boldsymbol{z}_k^r)
	}_{\displaystyle \boldsymbol{H}^r}.
\end{equation}
The equality follows by averaging across all clients at
\emph{each common point}:
\[
\frac{1}{N}\sum_{i=1}^{N}
\nabla F_i(\boldsymbol{z}_k^r)
=
\nabla F(\boldsymbol{z}_k^r).
\]
It holds regardless of the magnitude of
$\nabla F_i(\boldsymbol{z}_k^r)
-\nabla F(\boldsymbol{z}_k^r)$.
Accordingly, heterogeneity does not create an additional
client-trajectory term in the conditional mean under
full participation.

The path-averaged direction $\boldsymbol{H}^r$ is
generally distinct from
$\nabla F(\boldsymbol{x}^r)$. For $L$-smooth $F$,
\begin{equation}
	\label{eq:ctp_method_bias}
	\|\boldsymbol{H}^r-\nabla F(\boldsymbol{x}^r)\|
	\le
	\frac{L}{K}\sum_{k=0}^{K-1}
	t_k\|\boldsymbol{v}^r\|
	\le
	\frac{Lh}{2}\|\boldsymbol{v}^r\|.
\end{equation}
The prediction span therefore has a precise role:
it governs how far the algorithm probes along its previous
direction and how much displacement bias that probe can
introduce. When $h=0$, all queries coincide at
$\boldsymbol{x}^r$, giving synchronized mini-batch SGD
with $NK$ gradient evaluations per round.

\paragraph{A curvature interpretation.}
The shared path also provides a useful interpretation of
what the prediction changes. If $F$ is a quadratic with
constant Hessian $H$, then the uniform schedule in
\eqref{eq:ctp_query_points} yields the exact identity
\begin{equation}
	\label{eq:ctp_quadratic_insight}
	\boldsymbol{H}^r
	=
	\nabla F(\boldsymbol{x}^r)
	-\frac{h}{2}H\boldsymbol{v}^r.
\end{equation}
Thus, relative to the $h=0$ direction, the prediction
incorporates a Hessian--direction term \emph{through
	gradient evaluations}, without constructing a Hessian
or communicating extra state. When
$\boldsymbol{v}^r$ approximates the current global
gradient, positive-curvature directions are attenuated
by this term; for general non-convex objectives, its
effect depends on the local curvature and the accuracy
of $\boldsymbol{v}^r$. Equation~\eqref{eq:ctp_quadratic_insight}
is an interpretation of the update, not a claim of
unconditional acceleration.

Our convergence analysis takes $h=\gamma$ and controls
the resulting predictive bias jointly with stochastic
gradient noise. It establishes a full-participation
stationarity guarantee without a bounded-gradient or
bounded-client-dissimilarity assumption. Under partial
participation, the identity at each shared point holds
only after expectation over client sampling; the sampled
direction has an additional client-sampling variance.
The basic \texttt{CTP-FL} algorithm does not remove that
term.

\begin{algorithm}[tb]
	\small
	\caption{Common-Trajectory Predictive Federated Learning (CTP-FL)}
	\label{algorithm_CTP_FL}
	\begin{algorithmic}[1]
		\Require Server stepsize $\gamma$, prediction span $h$,
		communication rounds $R$, gradients per round $K$,
		number of clients $N$
		\State Initialize global model $\boldsymbol{x}^0$
		and prediction direction $\boldsymbol{v}^0=\boldsymbol{0}$
		\For{$r=0,\dots,R-1$}
		\State Server broadcasts $\boldsymbol{x}^r$
		\For{each client $i=1,\dots,N$ in parallel}
		\State Recover
		$\boldsymbol{v}^r
		\gets(\boldsymbol{x}^{r-1}-\boldsymbol{x}^r)/\gamma$
		if $r>0$; otherwise use $\boldsymbol{v}^0=\boldsymbol{0}$
		\State $\boldsymbol{u}_i^r\gets\boldsymbol{0}$
		\For{$k=0,\dots,K-1$}
		\State $t_k\gets kh/(K-1)$ if $K>1$;
		otherwise $t_0\gets0$
		\State \fcolorbox{LightBlue}{LightBlue}{$
			\boldsymbol{z}_k^r
			\gets\boldsymbol{x}^r-t_k\boldsymbol{v}^r$}
		\State Sample mini-batch $B_i^{r,k}$;
		$\boldsymbol{g}_i^{r,k}
		\gets\nabla F_i(\boldsymbol{z}_k^r;B_i^{r,k})$
		\State $\boldsymbol{u}_i^r
		\gets\boldsymbol{u}_i^r
		+\boldsymbol{g}_i^{r,k}/K$
		\EndFor
		\State Client $i$ sends $\boldsymbol{u}_i^r$
		to the server
		\EndFor
		\State \fcolorbox{LightBlue}{LightBlue}{$
			\boldsymbol{v}^{r+1}
			\gets\frac{1}{N}\sum_{i=1}^{N}\boldsymbol{u}_i^r$}
		\State $\boldsymbol{x}^{r+1}
		\gets\boldsymbol{x}^{r}
		-\gamma\boldsymbol{v}^{r+1}$
		\EndFor
		\Ensure Global model $\boldsymbol{x}^{R}$
	\end{algorithmic}
\end{algorithm}

\section{Theoretical Analysis}
\label{sec:theoretical_analysis}
\vspace{-2mm}

We analyze the non-convex objective
\begin{equation}
	\label{eq:ctp_objective}
	F(\boldsymbol{x})
	=
	\frac{1}{N}\sum_{i=1}^{N}F_i(\boldsymbol{x}),
	\qquad
	F_*=\inf_{\boldsymbol{x}\in\mathbb{R}^d}F(\boldsymbol{x})>-\infty.
\end{equation}
Our analysis separates two sources of error in the server direction:
stochastic-gradient noise and the displacement between the current
model and the shared predictive trajectory. Under partial
participation, client sampling introduces a third source of error.
We use the following standard assumptions
\citep{karimireddy2020scaffold,cheng2024momentum}.

\begin{assumption}[Smoothness]
	\label{ass:ctp_smooth_main}
	Each local objective $F_i$ is $L$-smooth. That is, for every
	$i\in[N]$ and $\boldsymbol{x},\boldsymbol{y}\in\mathbb{R}^d$,
	\begin{equation}
		\|\nabla F_i(\boldsymbol{x})
		-\nabla F_i(\boldsymbol{y})\|
		\le L\|\boldsymbol{x}-\boldsymbol{y}\|.
	\end{equation}
\end{assumption}

\begin{assumption}[Unbiased stochastic gradients]
	\label{ass:ctp_noise_main}
	For every client $i$ and every $\boldsymbol{x}$, its mini-batch
	gradient $g_i(\boldsymbol{x};\xi)$ satisfies
	\begin{equation}
		\mathbb{E}_{\xi}[g_i(\boldsymbol{x};\xi)]
		=\nabla F_i(\boldsymbol{x}),
		\qquad
		\mathbb{E}_{\xi}
		\|g_i(\boldsymbol{x};\xi)
		-\nabla F_i(\boldsymbol{x})\|^2
		\le \sigma^2.
	\end{equation}
	The mini-batches used at different query points, clients, and
	communication rounds are conditionally independent.
\end{assumption}

Neither assumption bounds individual gradients or the differences
$\nabla F_i(\boldsymbol{x})-\nabla F(\boldsymbol{x})$.

\paragraph{Predictive direction.}
Let $\mathcal F_r$ denote the history before the mini-batches of
round $r$ are sampled. For the shared query points
\[
\boldsymbol{z}_k^r
=\boldsymbol{x}^r-t_k\boldsymbol{v}^r,
\qquad
t_k=
\begin{cases}
	kh/(K-1),&K>1,\\
	0,&K=1,
\end{cases}
\]
define
\begin{equation}
	\label{eq:ctp_path_gradient_main}
	\boldsymbol{H}^r
	=
	\frac1K\sum_{k=0}^{K-1}
	\nabla F(\boldsymbol{z}_k^r).
\end{equation}
Under full participation, the server direction in
Algorithm~\ref{algorithm_CTP_FL} obeys
\begin{equation}
	\label{eq:ctp_direction_properties_main}
	\mathbb E[\boldsymbol{v}^{r+1}\mid\mathcal F_r]
	=\boldsymbol{H}^r,
	\qquad
	\mathbb E\!\left[
	\|\boldsymbol{v}^{r+1}-\boldsymbol{H}^r\|^2
	\mid\mathcal F_r
	\right]
	\le\frac{\sigma^2}{NK}.
\end{equation}
In particular, heterogeneous client gradients cancel at
\emph{each shared query point}. The resulting direction estimates
the average gradient along the predicted path, rather than the
gradient at $\boldsymbol{x}^r$. Smoothness bounds this predictive
bias by
\begin{equation}
	\label{eq:ctp_predictive_bias_main}
	\|\boldsymbol{H}^r-\nabla F(\boldsymbol{x}^r)\|
	\le
	\frac{L}{K}\sum_{k=0}^{K-1}
	t_k\|\boldsymbol{v}^r\|
	\le
	\frac{Lh}{2}\|\boldsymbol{v}^r\|.
\end{equation}
Thus, the path length $h$ controls a bias--variance trade-off
without introducing a client-dissimilarity constant.

\begin{theorem}[Full-participation convergence of \texttt{CTP-FL}]
	\label{thm:ctp_full_main_text}
	Suppose Assumptions~\ref{ass:ctp_smooth_main} and
	\ref{ass:ctp_noise_main} hold, and all $N$ clients participate
	in every round. Initialize $\boldsymbol{v}^0=\boldsymbol{0}$,
	set $h=\gamma$, and let
	\[
	\Delta=F(\boldsymbol{x}^0)-F_*,
	\qquad
	0<\gamma\le\frac{1}{4L}.
	\]
	Then Algorithm~\ref{algorithm_CTP_FL} satisfies
	\begin{equation}
		\label{eq:ctp_full_main_text}
		\boxed{
			\frac1R\sum_{r=0}^{R-1}
			\mathbb E\|\nabla F(\boldsymbol{x}^r)\|^2
			\le
			\frac{6\Delta}{\gamma R}
			+\frac{5L\gamma\sigma^2}{NK}.
		}
	\end{equation}
	If $\Delta>0$ and
	\begin{equation}
		\label{eq:ctp_gamma_main_text}
		\gamma=
		\left(
		4L+
		\sqrt{\frac{5L\sigma^2R}{6\Delta NK}}
		\right)^{-1},
	\end{equation}
	then
	\begin{equation}
		\label{eq:ctp_rate_main_text}
		\boxed{
			\frac1R\sum_{r=0}^{R-1}
			\mathbb E\|\nabla F(\boldsymbol{x}^r)\|^2
			\le
			2\sqrt{\frac{30L\Delta\sigma^2}{NKR}}
			+\frac{24L\Delta}{R}.
		}
	\end{equation}
\end{theorem}

Theorem~\ref{thm:ctp_full_main_text} establishes the same
$\mathcal O\!\left(
\sqrt{L\Delta\sigma^2/(NKR)}+L\Delta/R
\right)$ order as the full-participation guarantee of
FedAvg-M \citep{cheng2024momentum}. It does not imply a
strictly faster convergence rate. The methodological
difference lies in how the $K$ gradient evaluations are
allocated: FedAvg-M advances client-specific local
trajectories, whereas \texttt{CTP-FL} evaluates gradients
along a common trajectory. The proof is given in
Appendix~\ref{app:ctp_convergence}.

\paragraph{Scope under partial participation.}
For completeness, suppose a subset $\mathcal S_r$ of size
$S<N$ is sampled uniformly without replacement and the server
averages only the selected clients' directions. Define
\[
\boldsymbol{H}_i^r
=\frac1K\sum_{k=0}^{K-1}
\nabla F_i(\boldsymbol{z}_k^r),
\qquad
D_r
=
\frac1N\sum_{i=1}^{N}
\|\boldsymbol{H}_i^r-\boldsymbol{H}^r\|^2,
\]
and
\[
\rho_S=\frac{N-S}{S(N-1)}.
\]
Here $D_r$ is the client dispersion \emph{along the observed
	predictive trajectory}; the algorithm neither computes it nor
assumes that it is bounded.

\begin{theorem}[Partial-participation trajectory bound]
	\label{thm:ctp_partial_main_text}
	Under Assumptions~\ref{ass:ctp_smooth_main} and
	\ref{ass:ctp_noise_main}, uniform sampling without replacement,
	$h=\gamma$, and $0<\gamma\le1/(4L)$,
	\begin{equation}
		\label{eq:ctp_partial_main_text}
		\boxed{
			\begin{aligned}
				\frac1R\sum_{r=0}^{R-1}
				\mathbb E\|\nabla F(\boldsymbol{x}^r)\|^2
				\le\;&
				\frac{6\Delta}{\gamma R}
				+\frac{5L\gamma\sigma^2}{SK}\\
				&+
				\frac{5L\gamma\rho_S}{R}
				\sum_{r=0}^{R-1}\mathbb E[D_r].
			\end{aligned}
		}
	\end{equation}
\end{theorem}

The last term is the finite-population variance caused by
client sampling. Therefore,
Theorem~\ref{thm:ctp_partial_main_text} does not give a
uniform heterogeneity-independent rate when $S<N$.
This limitation is intrinsic to the uncorrected sampled
direction: even at a minimizer of $F$, individual client
gradients may be arbitrarily large and cancel only after
full aggregation. The proof and an explicit two-client
counterexample are given in Appendix~\ref{app:ctp_convergence}.

\begin{table*}[tb]
	\centering
	\caption{\small
		Test accuracy (\%) after 300 communication rounds on
		CIFAR-100 and Tiny-ImageNet. Each experiment uses
		$N=100$ clients, $S=10$ participating clients per round,
		mini-batch size 50, and $K=50$ gradient evaluations
		per participating client. Data are partitioned with
		Dirichlet concentrations $\alpha\in\{0.1,0.05\}$.
		Bold denotes the highest accuracy and underlining denotes
		the second-highest accuracy among the reported methods
		in each column.}
	\label{tab:ctp_vision}
	\setlength{\tabcolsep}{5pt}
	\renewcommand{\arraystretch}{1.12}
	\small
	\resizebox{\textwidth}{!}{
		\begin{tabular}{lcccccccc}
			\toprule
			\multirow{3}{*}{\textbf{Method}}
			& \multicolumn{4}{c}{\textbf{ResNet-18}}
			& \multicolumn{4}{c}{\textbf{ViT-Tiny}} \\
			\cmidrule(lr){2-5}\cmidrule(lr){6-9}
			& \multicolumn{2}{c}{\textbf{CIFAR-100}}
			& \multicolumn{2}{c}{\textbf{Tiny-ImageNet}}
			& \multicolumn{2}{c}{\textbf{CIFAR-100}}
			& \multicolumn{2}{c}{\textbf{Tiny-ImageNet}} \\
			\cmidrule(lr){2-3}\cmidrule(lr){4-5}
			\cmidrule(lr){6-7}\cmidrule(lr){8-9}
			& \textbf{Dir-0.1} & \textbf{Dir-0.05}
			& \textbf{Dir-0.1} & \textbf{Dir-0.05}
			& \textbf{Dir-0.1} & \textbf{Dir-0.05}
			& \textbf{Dir-0.1} & \textbf{Dir-0.05} \\
			\midrule
			FedAvg
			& 60.17 & 56.75 & 47.48 & 43.80
			& 27.24 & 23.42 & 15.68 & 14.05 \\
			SCAFFOLD
			& 60.69 & 56.43
			& \underline{47.76} & \underline{43.92}
			& 26.86 & 23.23 & 15.70 & 14.21 \\
			FedCM
			& \underline{66.61} & \underline{62.65}
			& 41.16 & 36.00
			& \underline{28.23} & \underline{25.74}
			& \underline{18.88} & \underline{18.15} \\
			\rowcolor{LightRed}
			\texttt{CTP-FL}
			& \textbf{69.25} & \textbf{64.16}
			& \textbf{55.62} & \textbf{51.81}
			& \textbf{51.16} & \textbf{47.55}
			& \textbf{34.32} & \textbf{31.33} \\
			\bottomrule
		\end{tabular}
	}
\end{table*}
\section{Experiments}
\label{sec:experiments}
\vspace{-2mm}

\paragraph{Experimental setup.}
We evaluate \texttt{CTP-FL} on CIFAR-100
\citep{krizhevsky2009learning} and Tiny-ImageNet
\citep{le2015tiny}, using ResNet-18
\citep{he2016deep} and ViT-Tiny
\citep{dosovitskiy2020image}.
To examine statistical heterogeneity, we partition the
training data among $N=100$ clients using a
Dirichlet distribution with concentration
$\alpha\in\{0.1,0.05\}$; the smaller value represents
a more concentrated allocation of classes across clients.
At each of 300 communication rounds, $S=10$ clients
participate, and each participating client computes
$K=50$ mini-batch gradients with batch size 50.
For \texttt{CTP-FL}, these gradients are evaluated
at the $K$ shared predictive points in
Eq.~\eqref{eq:ctp_query_points}; they are not
$K$ successive local model updates.

\paragraph{Baselines and comparison budget.}
We compare against FedAvg
\citep{mcmahan2017communication}, SCAFFOLD
\citep{karimireddy2020scaffold}, and FedCM
\citep{xu2021fedcm}. All reported methods use the
same number of communication rounds and participating
clients. The value $K=50$ fixes the number of
mini-batch gradient evaluations for \texttt{CTP-FL};
for local-update baselines it corresponds to the
number of local gradient steps. We report test accuracy
at the end of round 300 in Table~\ref{tab:ctp_vision}.

\paragraph{Main results.}
\texttt{CTP-FL} achieves the highest reported accuracy
in all eight settings of Table~\ref{tab:ctp_vision}.
On CIFAR-100 with ResNet-18, it improves over the
strongest listed baseline by 2.64 and 1.51 percentage
points under Dir-0.1 and Dir-0.05, respectively.
On Tiny-ImageNet with ResNet-18, the corresponding
gains are 7.86 and 7.89 points.
The differences are larger for ViT-Tiny:
22.93 and 21.81 points on CIFAR-100, and
15.44 and 13.18 points on Tiny-ImageNet.
These comparisons describe the results at round 300;
they do not, on their own, establish faster convergence
or identify the cause of the gains.

\paragraph{Theory--experiment scope.}
These experiments use partial client participation,
whereas Theorem~\ref{thm:ctp_full_main_text}
establishes a heterogeneity-independent rate under
full participation. Under partial participation,
Theorem~\ref{thm:ctp_partial_main_text} includes
a trajectory-dependent client-sampling term.
Consequently, the empirical performance in
Table~\ref{tab:ctp_vision} should not be presented
as verification of the full-participation rate.

\section{Conclusion}
\label{sec:conclusion}

We introduced \texttt{CTP-FL}, which uses the gradient evaluations
available between communication rounds to probe a shared predictive
trajectory. Because all clients evaluate gradients at the same
points, their full-participation average estimates the global
gradient along that trajectory, rather than combining gradients
measured at divergent local models. The prediction span makes the
resulting displacement bias explicit and controllable.

Under smoothness and bounded stochastic-gradient variance, we proved
a non-convex full-participation stationarity rate of
$\mathcal{O}\!\left(
\sqrt{L\Delta\sigma^2/(NKR)}+L\Delta/R
\right)$ without assuming bounded gradients or bounded
client-gradient dissimilarity. Each client computes $K$ mini-batch
gradients and uploads one model-sized vector per round. In the
reported partial-participation experiments, \texttt{CTP-FL} achieves
the highest final accuracy among the listed methods across
CIFAR-100 and Tiny-ImageNet with ResNet-18 and ViT-Tiny.

Our theory and experiments have distinct scopes. The
heterogeneity-independent rate applies to full participation;
sampling clients introduces an additional variance term that the
current algorithm does not eliminate. Further work should examine
control mechanisms for partial participation and determine when
probing a predictive trajectory improves on querying the current
model alone under matched computation and communication budgets.

\newpage
\bibliography{main}
\bibliographystyle{iclr2027_conference}

\appendix
\newpage
\section{Convergence Analysis of CTP-FL}
\label{app:ctp_convergence}

We consider
\begin{equation}
	\min_{\boldsymbol{x}\in\mathbb{R}^d}
	F(\boldsymbol{x})
	\coloneqq \frac{1}{N}\sum_{i=1}^{N}F_i(\boldsymbol{x}),
	\qquad
	F_* \coloneqq \inf_{\boldsymbol{x}}F(\boldsymbol{x})>-\infty.
\end{equation}
The following assumptions match the standard smoothness and stochastic
gradient assumptions used in FedAvg-M.

\begin{assumption}[Smoothness]
	\label{ass:ctp_smooth}
	For each client $i$, $F_i$ is $L$-smooth:
	\[
	\|\nabla F_i(\boldsymbol{x})-\nabla F_i(\boldsymbol{y})\|
	\le L\|\boldsymbol{x}-\boldsymbol{y}\|
	\quad
	\text{for all }\boldsymbol{x},\boldsymbol{y}\in\mathbb{R}^d.
	\]
\end{assumption}

\begin{assumption}[Stochastic gradients]
	\label{ass:ctp_noise}
	For each client $i$ and every $\boldsymbol{x}$,
	\[
	\mathbb{E}_{\xi}[g_i(\boldsymbol{x};\xi)]
	=\nabla F_i(\boldsymbol{x}),
	\qquad
	\mathbb{E}_{\xi}
	\|g_i(\boldsymbol{x};\xi)-\nabla F_i(\boldsymbol{x})\|^2
	\le \sigma^2.
	\]
	Fresh mini-batches are sampled independently across clients, local
	gradient evaluations, and communication rounds.
\end{assumption}

\paragraph{CTP-FL update.}
Initialize $\boldsymbol{v}^0=\boldsymbol{0}$. At round $r$, let
$\mathcal{S}_r$ contain $S$ clients, sampled uniformly without replacement,
independently of the mini-batches. Set
\begin{equation}
	\label{eq:ctp_locations}
	t_k=
	\begin{cases}
		\dfrac{k\gamma}{K-1}, & K\ge 2,\\[4pt]
		0, & K=1,
	\end{cases}
	\qquad
	\boldsymbol{z}_k^r
	=\boldsymbol{x}^r-t_k\boldsymbol{v}^r,
	\quad k=0,\ldots,K-1.
\end{equation}
Each selected client computes
\begin{equation}
	\label{eq:ctp_client}
	\boldsymbol{u}_i^r
	=\frac{1}{K}\sum_{k=0}^{K-1}
	g_i(\boldsymbol{z}_k^r;\xi_{i,k}^r).
\end{equation}
The server updates
\begin{equation}
	\label{eq:ctp_server}
	\boldsymbol{v}^{r+1}
	=\frac{1}{S}\sum_{i\in\mathcal{S}_r}\boldsymbol{u}_i^r,
	\qquad
	\boldsymbol{x}^{r+1}
	=\boldsymbol{x}^{r}-\gamma\boldsymbol{v}^{r+1}.
\end{equation}
Full participation corresponds to $S=N$.

Define
\begin{equation}
	\label{eq:ctp_averaged_gradients}
	\boldsymbol{h}_i^r
	\coloneqq
	\frac{1}{K}\sum_{k=0}^{K-1}
	\nabla F_i(\boldsymbol{z}_k^r),
	\qquad
	\boldsymbol{h}^r
	\coloneqq
	\frac{1}{N}\sum_{i=1}^{N}\boldsymbol{h}_i^r
	=
	\frac{1}{K}\sum_{k=0}^{K-1}
	\nabla F(\boldsymbol{z}_k^r).
\end{equation}
The predictive bias is
\[
\boldsymbol{b}^r
\coloneqq
\boldsymbol{h}^r-\nabla F(\boldsymbol{x}^r).
\]
For partial participation, define the \emph{observed trajectory
	dispersion}
\begin{equation}
	\label{eq:ctp_dispersion}
	D_r
	\coloneqq
	\frac{1}{N}\sum_{i=1}^{N}
	\|\boldsymbol{h}_i^r-\boldsymbol{h}^r\|^2.
\end{equation}
This is an analysis quantity, not an assumption or an input to CTP-FL.

\begin{lemma}[Bias and conditional variance]
	\label{lem:ctp_estimator}
	Under Assumptions~\ref{ass:ctp_smooth}--\ref{ass:ctp_noise},
	conditional on the history before client sampling at round $r$,
	\begin{equation}
		\label{eq:ctp_bias}
		\|\boldsymbol{b}^r\|^2
		\le
		\frac{L^2\gamma^2}{4}\|\boldsymbol{v}^r\|^2.
	\end{equation}
	Moreover,
	\begin{equation}
		\label{eq:ctp_conditional_mean}
		\mathbb{E}_r[\boldsymbol{v}^{r+1}]
		=\boldsymbol{h}^r,
	\end{equation}
	and
	\begin{equation}
		\label{eq:ctp_conditional_variance}
		\mathbb{E}_r
		\|\boldsymbol{v}^{r+1}-\boldsymbol{h}^r\|^2
		\le
		\frac{\sigma^2}{SK}
		+
		\frac{N-S}{S(N-1)}D_r,
	\end{equation}
	where the second term is defined to be zero when $S=N$.
\end{lemma}

\begin{proof}
	Since $F$ is $L$-smooth and
	$K^{-1}\sum_{k=0}^{K-1}t_k\le\gamma/2$,
	\begin{align}
		\|\boldsymbol{b}^r\|
		&\le
		\frac{1}{K}\sum_{k=0}^{K-1}
		\|\nabla F(\boldsymbol{z}_k^r)
		-\nabla F(\boldsymbol{x}^r)\| \notag\\
		&\le
		\frac{L}{K}\sum_{k=0}^{K-1}
		t_k\|\boldsymbol{v}^r\|
		\le
		\frac{L\gamma}{2}\|\boldsymbol{v}^r\|,
	\end{align}
	which proves~\eqref{eq:ctp_bias}.
	Unbiasedness of the stochastic gradients and uniform client sampling
	give~\eqref{eq:ctp_conditional_mean}.
	
	For~\eqref{eq:ctp_conditional_variance}, first condition on
	$\mathcal{S}_r$. The independent stochastic gradient errors have
	mean zero and their average has second moment at most
	$\sigma^2/(SK)$. The remaining variance is that of sampling without
	replacement from the finite population
	$\{\boldsymbol{h}_i^r\}_{i=1}^{N}$:
	\[
	\mathbb{E}_r
	\left\|
	\frac{1}{S}\sum_{i\in\mathcal{S}_r}
	\boldsymbol{h}_i^r-\boldsymbol{h}^r
	\right\|^2
	=
	\frac{N-S}{S(N-1)}
	\frac{1}{N}\sum_{i=1}^{N}
	\|\boldsymbol{h}_i^r-\boldsymbol{h}^r\|^2.
	\]
	The two errors have zero cross term, completing the proof.
\end{proof}

\begin{theorem}[Full-participation convergence]
	\label{thm:ctp_full}
	Suppose Assumptions~\ref{ass:ctp_smooth}--\ref{ass:ctp_noise}
	hold, all $N$ clients participate, and
	$0<\gamma\le 1/(4L)$. Let
	$\Delta\coloneqq F(\boldsymbol{x}^0)-F_*$. Then CTP-FL satisfies
	\begin{equation}
		\label{eq:ctp_full_rate}
		\boxed{
			\frac{1}{R}\sum_{r=0}^{R-1}
			\mathbb{E}\|\nabla F(\boldsymbol{x}^r)\|^2
			\le
			\frac{6\Delta}{\gamma R}
			+
			\frac{5L\gamma\sigma^2}{NK}.
		}
	\end{equation}
	In particular, no bounded-gradient or bounded-data-heterogeneity
	assumption is needed.
\end{theorem}

\begin{proof}
	Write
	\[
	a\coloneqq\frac{L^2\gamma^2}{4},
	\qquad
	Z_r\coloneqq
	\mathbb{E}_r
	\|\boldsymbol{v}^{r+1}-\boldsymbol{h}^r\|^2.
	\]
	Here $a\le 1/64$ and, under full participation,
	$Z_r\le \sigma^2/(NK)$.
	
	Smoothness of $F$, the server update, and
	$\mathbb{E}_r[\boldsymbol{v}^{r+1}]
	=\nabla F(\boldsymbol{x}^r)+\boldsymbol{b}^r$ imply
	\begin{align}
		\mathbb{E}_r[F(\boldsymbol{x}^{r+1})]
		\le\;&
		F(\boldsymbol{x}^r)
		-\gamma
		\left\langle
		\nabla F(\boldsymbol{x}^r),
		\nabla F(\boldsymbol{x}^r)+\boldsymbol{b}^r
		\right\rangle \notag\\
		&+
		\frac{L\gamma^2}{2}
		\left(
		\|\nabla F(\boldsymbol{x}^r)+\boldsymbol{b}^r\|^2
		+Z_r
		\right).
		\label{eq:ctp_descent_start}
	\end{align}
	Using
	$-\langle \boldsymbol{p},\boldsymbol{q}\rangle
	\le \|\boldsymbol{p}\|^2/4+\|\boldsymbol{q}\|^2$,
	$\|\boldsymbol{p}+\boldsymbol{q}\|^2
	\le 2\|\boldsymbol{p}\|^2+2\|\boldsymbol{q}\|^2$,
	and $L\gamma\le 1/4$ gives the conservative bound
	\begin{equation}
		\label{eq:ctp_descent}
		\mathbb{E}_r[F(\boldsymbol{x}^{r+1})]
		\le
		F(\boldsymbol{x}^r)
		-\frac{\gamma}{4}
		\|\nabla F(\boldsymbol{x}^r)\|^2
		+\frac{3\gamma}{2}\|\boldsymbol{b}^r\|^2
		+\frac{L\gamma^2}{2}Z_r.
	\end{equation}
	
	For compactness, set
	\[
	A_R
	\coloneqq
	\sum_{r=0}^{R-1}
	\mathbb{E}\|\nabla F(\boldsymbol{x}^r)\|^2,
	\quad
	V_R
	\coloneqq
	\sum_{r=0}^{R-1}
	\mathbb{E}\|\boldsymbol{v}^r\|^2,
	\quad
	\mathcal{Z}_R
	\coloneqq
	\sum_{r=0}^{R-1}\mathbb{E}[Z_r].
	\]
	Summing~\eqref{eq:ctp_descent}, applying
	$F(\boldsymbol{x}^R)\ge F_*$, and using
	Lemma~\ref{lem:ctp_estimator}, we obtain
	\begin{equation}
		\label{eq:ctp_A_recursion}
		A_R
		\le
		\frac{4\Delta}{\gamma}
		+6aV_R
		+2L\gamma\mathcal{Z}_R.
	\end{equation}
	
	Next, $\boldsymbol{v}^0=\boldsymbol{0}$ and, for $r\ge1$,
	\[
	\boldsymbol{v}^r
	=
	\nabla F(\boldsymbol{x}^{r-1})
	+\boldsymbol{b}^{r-1}
	+
	\bigl(\boldsymbol{v}^r-\boldsymbol{h}^{r-1}\bigr).
	\]
	The final parenthesized term has conditional mean zero.
	Consequently,
	\begin{align}
		\mathbb{E}\|\boldsymbol{v}^r\|^2
		&\le
		2\mathbb{E}\|\nabla F(\boldsymbol{x}^{r-1})\|^2
		+2\mathbb{E}\|\boldsymbol{b}^{r-1}\|^2
		+2\mathbb{E}[Z_{r-1}] \notag\\
		&\le
		2\mathbb{E}\|\nabla F(\boldsymbol{x}^{r-1})\|^2
		+2a\mathbb{E}\|\boldsymbol{v}^{r-1}\|^2
		+2\mathbb{E}[Z_{r-1}].
	\end{align}
	Summing over $r$ and extending nonnegative sums yields
	\begin{equation}
		\label{eq:ctp_V_recursion}
		V_R\le 2A_R+2aV_R+2\mathcal{Z}_R.
	\end{equation}
	Since $a\le1/64$,
	\[
	V_R
	\le
	\frac{2}{1-2a}(A_R+\mathcal{Z}_R)
	\le 3(A_R+\mathcal{Z}_R).
	\]
	Substituting this into~\eqref{eq:ctp_A_recursion} gives
	\[
	(1-18a)A_R
	\le
	\frac{4\Delta}{\gamma}
	+(18a+2L\gamma)\mathcal{Z}_R.
	\]
	Because $a\le1/64$,
	\begin{equation}
		\label{eq:ctp_master}
		A_R
		\le
		\frac{6\Delta}{\gamma}
		+(26a+3L\gamma)\mathcal{Z}_R.
	\end{equation}
	Finally,
	$\mathcal{Z}_R\le R\sigma^2/(NK)$ and
	$a=(L\gamma)^2/4\le L\gamma/16$, so that
	$26a+3L\gamma<5L\gamma$.
	Substitution into~\eqref{eq:ctp_master} proves
	\eqref{eq:ctp_full_rate}.
\end{proof}

\begin{theorem}[Partial-participation trajectory bound]
	\label{thm:ctp_partial}
	Suppose Assumptions~\ref{ass:ctp_smooth}--\ref{ass:ctp_noise}
	hold, $1\le S<N$, and $0<\gamma\le1/(4L)$.
	With $D_r$ defined in~\eqref{eq:ctp_dispersion},
	\begin{equation}
		\label{eq:ctp_partial_rate}
		\boxed{
			\begin{aligned}
				\frac{1}{R}\sum_{r=0}^{R-1}
				\mathbb{E}\|\nabla F(\boldsymbol{x}^r)\|^2
				\le\;&
				\frac{6\Delta}{\gamma R}
				+\frac{5L\gamma\sigma^2}{SK}\\
				&+
				\frac{5L\gamma(N-S)}{S(N-1)R}
				\sum_{r=0}^{R-1}\mathbb{E}[D_r].
			\end{aligned}
		}
	\end{equation}
	This theorem imposes no bound on $D_r$. Accordingly,
	\eqref{eq:ctp_partial_rate} is a valid trajectory-dependent
	inequality, but does not by itself guarantee a uniform
	heterogeneity-independent convergence rate.
\end{theorem}

\begin{proof}
	The proof of Theorem~\ref{thm:ctp_full} up to
	\eqref{eq:ctp_master} applies unchanged.
	By Lemma~\ref{lem:ctp_estimator},
	\[
	\mathcal{Z}_R
	\le
	\frac{R\sigma^2}{SK}
	+
	\frac{N-S}{S(N-1)}
	\sum_{r=0}^{R-1}\mathbb{E}[D_r].
	\]
	Substitute this bound into~\eqref{eq:ctp_master} and use
	$26a+3L\gamma<5L\gamma$.
\end{proof}

\begin{proposition}[Why client sampling cannot be ignored]
	\label{prop:ctp_impossibility}
	For the CTP-FL update~\eqref{eq:ctp_server} with $S<N$,
	Assumptions~\ref{ass:ctp_smooth}--\ref{ass:ctp_noise}
	alone cannot yield a uniform stationarity bound depending
	only on $\Delta,L,\sigma,N,S,K,R$, and $\gamma$ that vanishes
	when $\Delta=\sigma=0$.
\end{proposition}

\begin{proof}
	Take $N=2$, $S=1$, $d=1$, and exact gradients
	($\sigma=0$). For any $M>0$, define
	\[
	F_1(x)=\frac{L}{2}x^2+Mx,
	\qquad
	F_2(x)=\frac{L}{2}x^2-Mx.
	\]
	Then $F(x)=Lx^2/2$. Initialize $x^0=0$ and $v^0=0$.
	Thus $x^0$ is the global minimizer and $\Delta=0$.
	At round zero, every predictive location equals zero,
	regardless of $K$. The single sampled client returns
	$v^1=M$ or $v^1=-M$, each with probability $1/2$.
	Hence
	\[
	x^1=-\gamma v^1,
	\qquad
	\mathbb{E}\|\nabla F(x^1)\|^2
	=L^2\gamma^2M^2.
	\]
	For any $R\ge2$, the average stationarity measure
	therefore contains the strictly positive term
	$L^2\gamma^2M^2/R$, which can be made arbitrarily large
	by increasing $M$, although $\Delta=\sigma=0$.
\end{proof}

\begin{theorem}[Convergence of CTP-FL under full participation]
	\label{thm:ctp_full_main}
	Let
	\[
	F(\boldsymbol{x})
	=\frac{1}{N}\sum_{i=1}^{N}F_i(\boldsymbol{x}),
	\qquad
	\Delta=F(\boldsymbol{x}^0)-F_*,
	\qquad
	F_*=\inf_{\boldsymbol{x}}F(\boldsymbol{x})>-\infty.
	\]
	Suppose that every $F_i$ is $L$-smooth and that the stochastic
	gradients are unbiased, independent across clients and gradient
	evaluations, and have variance at most $\sigma^2$. All $N$ clients
	participate in every round.
	
	Run CTP-FL with $\boldsymbol{v}^0=\boldsymbol{0}$ and predictive
	locations
	\[
	\boldsymbol{z}_k^r
	=
	\boldsymbol{x}^r-t_k\boldsymbol{v}^r,
	\qquad
	t_k=
	\begin{cases}
		k\gamma/(K-1), & K\ge2,\\
		0, & K=1.
	\end{cases}
	\]
	If $0<\gamma\le 1/(4L)$, then
	\begin{equation}
		\label{eq:ctp_general_rate}
		\frac{1}{R}\sum_{r=0}^{R-1}
		\mathbb{E}\|\nabla F(\boldsymbol{x}^r)\|^2
		\le
		\frac{6\Delta}{\gamma R}
		+
		\frac{5L\gamma\sigma^2}{NK}.
	\end{equation}
	In particular, for $\Delta>0$, choose
	\begin{equation}
		\label{eq:ctp_stepsize}
		\boxed{
			\gamma
			=
			\left(
			4L+
			\sqrt{\frac{5L\sigma^2 R}{6\Delta NK}}
			\right)^{-1}.
		}
	\end{equation}
	Then
	\begin{equation}
		\label{eq:ctp_optimized_rate}
		\boxed{
			\frac{1}{R}\sum_{r=0}^{R-1}
			\mathbb{E}\|\nabla F(\boldsymbol{x}^r)\|^2
			\le
			2\sqrt{\frac{30L\Delta\sigma^2}{NKR}}
			+
			\frac{24L\Delta}{R}
			=
			\mathcal{O}\!\left(
			\sqrt{\frac{L\Delta\sigma^2}{NKR}}
			+\frac{L\Delta}{R}
			\right).
		}
	\end{equation}
	No bounded-gradient or bounded-data-heterogeneity assumption
	is imposed.
\end{theorem}

\begin{proof}
	Let $\mathcal{F}_r$ denote the history before the fresh
	mini-batches of round $r$ are sampled. Define
	\[
	\boldsymbol{h}^r
	=
	\frac{1}{K}\sum_{k=0}^{K-1}
	\nabla F(\boldsymbol{z}_k^r),
	\qquad
	\boldsymbol{b}^r
	=
	\boldsymbol{h}^r-\nabla F(\boldsymbol{x}^r),
	\qquad
	a=\frac{L^2\gamma^2}{4}.
	\]
	Because all clients evaluate their gradients at the same
	$\boldsymbol{z}_k^r$, unbiasedness and independence give
	\begin{equation}
		\label{eq:ctp_proof_moments}
		\mathbb{E}[\boldsymbol{v}^{r+1}\mid\mathcal{F}_r]
		=\boldsymbol{h}^r,
		\qquad
		\mathbb{E}\!\left[
		\|\boldsymbol{v}^{r+1}-\boldsymbol{h}^r\|^2
		\mid\mathcal{F}_r
		\right]
		\le
		\frac{\sigma^2}{NK}.
	\end{equation}
	Smoothness and
	$K^{-1}\sum_{k=0}^{K-1}t_k\le\gamma/2$ yield
	\begin{equation}
		\label{eq:ctp_proof_bias}
		\|\boldsymbol{b}^r\|
		\le
		\frac{L}{K}\sum_{k=0}^{K-1}
		t_k\|\boldsymbol{v}^r\|
		\le
		\frac{L\gamma}{2}\|\boldsymbol{v}^r\|,
		\qquad
		\|\boldsymbol{b}^r\|^2
		\le a\|\boldsymbol{v}^r\|^2.
	\end{equation}
	
	Since
	$\boldsymbol{x}^{r+1}
	=\boldsymbol{x}^r-\gamma\boldsymbol{v}^{r+1}$,
	the descent lemma and~\eqref{eq:ctp_proof_moments} imply
	\begin{align}
		\mathbb{E}[F(\boldsymbol{x}^{r+1})\mid\mathcal{F}_r]
		\le\;&
		F(\boldsymbol{x}^r)
		-\gamma
		\left\langle
		\nabla F(\boldsymbol{x}^r),
		\nabla F(\boldsymbol{x}^r)+\boldsymbol{b}^r
		\right\rangle
		\nonumber\\
		&+
		\frac{L\gamma^2}{2}
		\left(
		\|\nabla F(\boldsymbol{x}^r)+\boldsymbol{b}^r\|^2
		+\frac{\sigma^2}{NK}
		\right).
	\end{align}
	Apply
	\[
	-\langle\boldsymbol{p},\boldsymbol{q}\rangle
	\le
	\frac14\|\boldsymbol{p}\|^2+\|\boldsymbol{q}\|^2,
	\qquad
	\|\boldsymbol{p}+\boldsymbol{q}\|^2
	\le
	2\|\boldsymbol{p}\|^2+2\|\boldsymbol{q}\|^2.
	\]
	Because $L\gamma\le1/4$, we obtain
	\begin{equation}
		\label{eq:ctp_proof_descent}
		\mathbb{E}[F(\boldsymbol{x}^{r+1})\mid\mathcal{F}_r]
		\le
		F(\boldsymbol{x}^r)
		-\frac{\gamma}{4}
		\|\nabla F(\boldsymbol{x}^r)\|^2
		+\frac{3\gamma}{2}
		\|\boldsymbol{b}^r\|^2
		+\frac{L\gamma^2\sigma^2}{2NK}.
	\end{equation}
	
	Set
	\[
	A_R
	=
	\sum_{r=0}^{R-1}
	\mathbb{E}\|\nabla F(\boldsymbol{x}^r)\|^2,
	\qquad
	V_R
	=
	\sum_{r=0}^{R-1}
	\mathbb{E}\|\boldsymbol{v}^r\|^2,
	\qquad
	Q_R=\frac{R\sigma^2}{NK}.
	\]
	Summing~\eqref{eq:ctp_proof_descent}, using
	$F(\boldsymbol{x}^R)\ge F_*$, and applying
	\eqref{eq:ctp_proof_bias} give
	\begin{equation}
		\label{eq:ctp_proof_A}
		A_R
		\le
		\frac{4\Delta}{\gamma}
		+6aV_R+2L\gamma Q_R.
	\end{equation}
	Furthermore, $\boldsymbol{v}^0=\boldsymbol{0}$ and
	\[
	\boldsymbol{v}^r
	=
	\nabla F(\boldsymbol{x}^{r-1})
	+\boldsymbol{b}^{r-1}
	+
	(\boldsymbol{v}^r-\boldsymbol{h}^{r-1})
	\quad (r\ge1).
	\]
	The last term has conditional mean zero. Thus
	\begin{equation}
		\label{eq:ctp_proof_V}
		V_R
		\le
		2A_R+2aV_R+2Q_R.
	\end{equation}
	Since $\gamma\le1/(4L)$, we have $a\le1/64$.
	Therefore~\eqref{eq:ctp_proof_V} implies
	\[
	V_R
	\le
	\frac{2}{1-2a}(A_R+Q_R)
	\le
	3(A_R+Q_R).
	\]
	Substituting into~\eqref{eq:ctp_proof_A} gives
	\[
	(1-18a)A_R
	\le
	\frac{4\Delta}{\gamma}
	+(18a+2L\gamma)Q_R.
	\]
	As $a\le1/64$,
	\[
	A_R
	\le
	\frac{6\Delta}{\gamma}
	+(26a+3L\gamma)Q_R.
	\]
	Finally, $a=(L\gamma)^2/4\le L\gamma/16$, so
	$26a+3L\gamma<5L\gamma$. This proves
	\eqref{eq:ctp_general_rate}.
	
	For the parameter choice~\eqref{eq:ctp_stepsize},
	$\gamma\le1/(4L)$ and
	\[
	\frac{6\Delta}{\gamma R}
	=
	\frac{24L\Delta}{R}
	+
	\sqrt{\frac{30L\Delta\sigma^2}{NKR}}.
	\]
	Also,
	\[
	\frac{5L\gamma\sigma^2}{NK}
	\le
	\sqrt{\frac{30L\Delta\sigma^2}{NKR}},
	\]
	which proves~\eqref{eq:ctp_optimized_rate}.
\end{proof}

\begin{theorem}[CTP-FL under partial participation]
	\label{thm:ctp_partial_main}
	Under the assumptions of Theorem~\ref{thm:ctp_full_main},
	suppose each round samples $S<N$ clients uniformly without
	replacement. Define
	\[
	\boldsymbol{h}_i^r
	=
	\frac1K\sum_{k=0}^{K-1}
	\nabla F_i(\boldsymbol{z}_k^r),
	\qquad
	\boldsymbol{h}^r
	=
	\frac1N\sum_{i=1}^{N}\boldsymbol{h}_i^r,
	\]
	and
	\[
	D_r
	=
	\frac1N\sum_{i=1}^{N}
	\|\boldsymbol{h}_i^r-\boldsymbol{h}^r\|^2.
	\]
	For $0<\gamma\le1/(4L)$,
	\begin{equation}
		\label{eq:ctp_partial_main}
		\boxed{
			\begin{aligned}
				\frac1R\sum_{r=0}^{R-1}
				\mathbb{E}\|\nabla F(\boldsymbol{x}^r)\|^2
				\le\;&
				\frac{6\Delta}{\gamma R}
				+\frac{5L\gamma\sigma^2}{SK}\\
				&+
				\frac{5L\gamma(N-S)}{S(N-1)R}
				\sum_{r=0}^{R-1}\mathbb{E}[D_r].
			\end{aligned}
		}
	\end{equation}
	No uniform upper bound on $D_r$ is assumed.
\end{theorem}

\begin{proof}
	Conditional on the history before client sampling,
	uniform sampling gives
	$\mathbb{E}_r[\boldsymbol{v}^{r+1}]
	=\boldsymbol{h}^r$. The finite-population variance identity
	and independent mini-batches give
	\[
	\mathbb{E}_r
	\|\boldsymbol{v}^{r+1}-\boldsymbol{h}^r\|^2
	\le
	\frac{\sigma^2}{SK}
	+
	\frac{N-S}{S(N-1)}D_r.
	\]
	Repeat the proof of
	Theorem~\ref{thm:ctp_full_main}, replacing
	$R\sigma^2/(NK)$ throughout by
	\[
	\frac{R\sigma^2}{SK}
	+
	\frac{N-S}{S(N-1)}
	\sum_{r=0}^{R-1}\mathbb{E}[D_r].
	\]
	The same bias and descent bounds yield
	\eqref{eq:ctp_partial_main}.
\end{proof}

\section{Related Work}
\label{sec:related_work}

\paragraph{Communication-efficient federated optimization.}
Communication-efficient federated learning commonly allows clients to perform
multiple local updates between aggregation rounds. FedBCGD reduces the amount
of information transmitted in each round by updating and communicating
parameter blocks, and further incorporates drift control and variance
reduction in its accelerated variant~\citep{liu2024fedbcgd}. This approach
addresses the size of each message. In contrast, \texttt{CTP-FL} communicates
one model-sized vector per participating client and changes where local
gradients are evaluated: clients query the same server-defined predictive
trajectory, so their updates can be averaged without first combining models
that have followed different local trajectories.

\paragraph{Alignment under heterogeneous data.}
Several recent methods study different forms of local--global misalignment.
FedSWA and FedMoSWA use stochastic weight averaging and momentum-based control
to improve generalization under highly heterogeneous data~\citep{liu2025fedswa}.
FedNSAM examines the mismatch between local and global flatness and uses a
global Nesterov direction to improve their consistency~\citep{liu2025fednsam}.
These methods primarily target the properties of the resulting solution,
including flatness and generalization. Our focus is the geometry of gradient
evaluation during a communication round: when all clients evaluate at common
points, averaging their gradients estimates the gradient of the global
objective at those points, irrespective of how different the individual
client gradients are.

\paragraph{Federated adaptive and structured optimizers.}
FedAdamW combines local correction, decoupled weight decay, and aggregation
of second-moment estimates for federated large-model
training~\citep{liu2026fedadamw}. FedMuon exploits matrix orthogonalization
and local--global alignment to improve federated optimization of
matrix-structured parameters~\citep{liu2025fedmuon}. FedPAC identifies
preconditioner drift as a source of instability when local second-order
optimizers induce incompatible client geometries, and proposes
preconditioner alignment and update correction~\citep{liu2026fedpac}.
Unlike these optimizer-specific mechanisms, \texttt{CTP-FL} applies to
stochastic-gradient evaluations without transmitting moments or
preconditioners. Its common trajectory aligns the \emph{locations} of
gradient evaluation rather than optimizer states.

\paragraph{Global flatness and privacy.}
DP-FedPGN encourages globally flat solutions in client-level differentially
private federated learning through a global gradient-norm
penalty~\citep{liu2025dpfedpgn}. Its objective and privacy accounting are
different from ours. We cite it because it likewise illustrates that a
quantity defined by the global objective need not be faithfully represented
by independently optimized local objectives.

\end{document}